\documentclass[letterpaper]{article}
\usepackage{aaai2027}
\usepackage[utf8]{inputenc}
\usepackage{natbib}
\usepackage{appendix}
\usepackage{graphicx}
\usepackage{tikz}
\usetikzlibrary{arrows.meta,positioning}
\usepackage{booktabs}
\usepackage{amsmath,amssymb,amsthm}
\usepackage{algorithm}
\usepackage{algorithmic}
\usepackage{enumitem}
\usepackage{multirow}
\usepackage{array}
\usepackage{tabularx}
\usepackage{makecell}
\usepackage{pifont}
\usepackage{xcolor}
\usepackage{colortbl}

\definecolor{bestgreen}{RGB}{214,238,214}   
\definecolor{secondblue}{RGB}{215,232,247}  
\definecolor{ouryellow}{RGB}{253,247,214}   
\newcommand{\ourrow}{\rowcolor{ouryellow}}              
\definecolor{measurecol}{RGB}{33,124,124}    
\definecolor{derivecol}{RGB}{197,118,26}     
\definecolor{traincol}{RGB}{52,140,58}        
\definecolor{stepcol}{RGB}{40,72,120}         
\nocopyright
\newcommand{\method}{\textsc{MotifRole-Diff}}
\newtheorem{theorem}{Theorem}
\newtheorem{lemma}{Lemma}
\newtheorem{corollary}{Corollary}
\newtheorem{proposition}{Proposition}
\theoremstyle{definition}

\begin{document}
\title{\method{}: Risk-Optimal Role-Aware Corruption\\ for Masked Molecular Graph Diffusion}
\author{
    Tasfia Nuzhat Ornee,
    Elias Hossain,
    Ivan Garibay,
    Niloofar Yousefi
}
\affiliations{
    University of Central Florida\\
    tasfia.ornee@ucf.edu
}

\maketitle


\begin{abstract}
Masked discrete diffusion for molecular graph generation typically applies a uniform corruption schedule to all tokens in a lossless graph-to-sequence representation, implicitly treating structurally heterogeneous molecular components as equally difficult and equally important to reconstruct. However, different molecular graph token roles exhibit substantial variation in denoising difficulty and their influence on the decoded molecule, motivating role-specific corruption strategies. We introduce \textsc{MotifRole-Diff}, a role-aware corruption process that allocates masking rates according to empirically measured denoising difficulty and graph-level perturbation impact while preserving the model architecture, clean sequence space, and lossless molecular-graph decoder. We formulate schedule selection as the risk-optimal allocation of a fixed masking budget across token roles. Our theorem characterizes optimality for the modeled role-weighted residual risk, while downstream generation performance is evaluated empirically. Under matched architecture, training budget, and sampling compute, \textsc{MotifRole-Diff} improves validity on QM9 from 0.905 to 0.944 while reducing FCD from 1.701 to 1.609, and on MOSES improves validity from 0.920 to 0.938 while reducing FCD from 2.125 to 1.850. Role-wise diagnostics further show improved reconstruction across molecular graph token categories. Together, these matched-compute results indicate that structurally informed corruption is a more effective masking strategy than uniform schedules for serialized molecular graph diffusion.
\end{abstract}

\section{Introduction}

Generating valid, novel, and diverse molecular graphs is a fundamental problem in computational molecular discovery, and diffusion models have become one of the strongest paradigms for molecular graph generation~\cite{jo2022gdss,vignac2023digress,chen2023edge}. Masked discrete language modeling (MDLM)~\cite{sahoo2024mdlm}, built on discrete denoising diffusion~\cite{austin2021d3pm,campbell2022continuous}, provides a scalable formulation by serializing molecular graphs into token sequences, denoising them in parallel, and reconstructing the original graph through a lossless decoder, combining exact reconstruction with the efficiency of sequence modeling.

This approach depends fundamentally on graph serialization. Prior work has shown that serialization order influences sequence-model efficiency~\cite{diamant2023bandwidth,jang2024geel}, while hierarchical molecular generators demonstrate the importance of preserving chemically meaningful motifs~\cite{jin2018jtvae,jin2020hierarchical}. Existing masked discrete diffusion models, however, apply an identical corruption schedule to every token position~\cite{austin2021d3pm,sahoo2024mdlm}, implicitly assuming that all serialized tokens are equally difficult to reconstruct and equally important for recovering the molecular graph.

That assumption is often incorrect. Serialized molecular graphs contain structurally distinct token roles: boundary markers, serialization syntax, motif-interior atoms and bonds, and interface tokens connecting substructures. Errors on these roles are not equally consequential, since some are reconstructed reliably with little effect on the decoded graph while others are harder to recover or disproportionately disrupt validity and connectivity, yet existing graph-to-sequence diffusion methods do not ask whether corruption should depend on these differences. We call this property \emph{role heterogeneity}, and hypothesize that masked molecular diffusion is therefore not only \emph{order-sensitive} but also \emph{role-sensitive}: uniform corruption is appropriate only when roles have comparable structural importance, and otherwise corruption should follow measured structural risk.

To investigate this hypothesis, we introduce \method{}, a role-sensitive framework for masked discrete molecular diffusion. We extend the lossless SENT representation~\cite{chen2026flatten} with the motif-aware serialization \emph{mSENT}, exposing four structural token roles while preserving decoder compatibility. We quantify each role's reconstruction difficulty and graph-level consequence using controlled perturbation experiments, then derive a role-aware corruption process by allocating a fixed masking budget according to measured structural risk. The resulting absorbing diffusion formulation generalizes standard MDLM and reduces to uniform corruption when all token roles possess identical structural criticality.

Our hypothesis gives rise to four research questions. \textbf{(RQ1)} Do serialized token roles differ significantly in reconstruction difficulty and graph-level structural consequence? \textbf{(RQ2)} Can role-aware corruption be derived as a principled constrained allocation of a fixed masking budget? \textbf{(RQ3)} Does empirically informed role-aware corruption outperform uniform masking, and do the gains stem from informed structural allocation rather than mere non-uniformity? \textbf{(RQ4)} Does motif-aware serialization improve motif locality while preserving the lossless decoding guarantees of SENT?

Throughout, architecture, optimization, compute, and expected masking rate are held constant so that differences arise only from serialization and corruption design. Our contributions are fourfold:

\begin{enumerate}[leftmargin=*,itemsep=1pt,topsep=2pt]

\item \textbf{Characterizing role heterogeneity in serialized molecular diffusion (RQ1).}
We identify structural token roles as an overlooked property of graph serialization and show that they differ systematically in reconstruction difficulty and graph-level structural consequence.

\item \textbf{A principled role-aware corruption process (RQ2).}
We formulate corruption scheduling as constrained structural-risk minimization under a fixed masking budget, yielding a role-aware absorbing diffusion process that generalizes standard MDLM and recovers uniform corruption as a special case.

\item \textbf{Controlled evaluation of role-aware corruption (RQ3).}
Under matched settings, \method{} improves generation quality over uniform MDLM at equal compute; because the schedule is derived from measured criticalities rather than set arbitrarily, the gain reflects informed structural allocation rather than non-uniform corruption alone.

\item \textbf{A motif-aware role-preserving serialization (RQ4).}
We introduce mSENT, a motif-aware extension of SENT that preserves the original lossless decoder while improving motif locality and exposing four interpretable structural token roles.

\end{enumerate}

\section{Preliminaries and Problem Formulation}

\paragraph{Notation.}
A molecular graph $G=(V,E,X,A)$ has atoms $V$, bonds $E$, features $X$, and adjacency
$A$. A lossless serialization $T(\cdot\,;\pi)$ maps $G$ under a traversal order $\pi$ to
a token sequence $z^{\pi}=T(G;\pi)$ with decoder $D(z^{\pi})\cong G$ and
$|z^{\pi}|=O(|V|+|E|)$. The motif-aware tokenizer assigns atoms to motif blocks
$\mu:V\to\{1,\dots,K\}$ and tags each token with a role
$r_i\in\mathcal R=\{\text{special},\text{syntax},\text{interior},\text{interface}\}$
(Table~\ref{tab:roles}); $\pi_r$ denotes the fraction of role-$r$ tokens.

\begin{table}[!htb]
\centering
\footnotesize
\setlength{\tabcolsep}{4pt}
\caption{Token roles induced by motif-aware serialization.}
\label{tab:roles}
\begin{tabularx}{\columnwidth}{@{}l
  >{\raggedright\arraybackslash}X
  >{\raggedright\arraybackslash}X@{}}
\toprule
Role & Meaning & Consequence of error \\
\midrule
Special   & control / \texttt{[MASK]} tokens & invalid boundary or padding \\
Syntax    & grammar, reset, delimiters       & malformed or undecodable graph \\
Interior  & atoms/bonds within one motif     & local motif corruption; occasional scaffold / atom-bond change \\
Interface & cross-motif attachment tokens    & attachment ambiguity and possible connectivity error; highest single-token perturbation impact \\
\bottomrule
\end{tabularx}
\end{table}

\paragraph{Masked discrete diffusion.}
An absorbing forward process replaces each token of a clean sequence $z_0$ by
\texttt{[MASK]} with probability $\alpha_t$ at diffusion time $t$, and a denoiser
$p_\theta$ reconstructs the masked tokens under the absorbing-state ELBO,
\begin{equation}
\mathcal{L}_{\text{MDLM}}=\mathbb{E}_{z_0,t,M_t}\!\Big[w(t)\!\!\sum_{i\in M_t}\!
-\log p_\theta\big(z_0^{(i)}\mid z_t,t\big)\Big].
\label{eq:mdlm}
\end{equation}
The rate $\alpha_t$ is token-agnostic: every position is corrupted identically regardless
of role, and a confidence-based sampler reveals tokens in the reverse process.

\paragraph{Role exposure and budget.}
For a role-dependent schedule $\{\alpha_t^{(r)}\}$, let
$\rho_r=\mathbb E_t[\alpha_t^{(r)}]$ be the average exposure of role $r$ and
$\bar\rho=\sum_{r}\pi_r\rho_r$ the total masking budget. Uniform MDLM is the case
$\rho_r\equiv\bar\rho$. Fixing $\bar\rho$ equalizes the expected number of masked tokens,
so schedules sharing $\bar\rho$ are matched in budget and, under a shared backbone, in
compute.

\paragraph{Problem statement.}
Given $\mathcal R$, the frequencies $\{\pi_r\}$, and a fixed budget $\bar\rho$, we seek
the role-dependent schedule minimizing the expected graph-level reconstruction error at
equal budget,
\begin{equation}
\begin{aligned}
\min_{\rho}\ & \mathbb E[\Delta_{\mathrm{graph}}(\rho)]\\
\text{s.t.}\ & \textstyle\sum_{r}\pi_r\rho_r=\bar\rho,\quad
\rho_{\min}\le\rho_r\le\rho_{\max}.
\end{aligned}
\label{eq:problem}
\end{equation}
This reallocates a fixed budget across roles rather than adding noise or capacity; the
Theoretical Analysis makes $\mathbb E[\Delta_{\mathrm{graph}}]$ concrete and derives the
risk-optimal allocation.

\section{Role Sensitivity: A Diagnostic}
On a lightweight \emph{probe} denoiser $p_{\theta_0}$ trained with uniform masking, we
measure for each role $r$ the difficulty
$D_r=\mathbb{E}_{i:r_i=r}[-\log p_{\theta_0}(z_0^{(i)}\mid z_t,t)]$ and top-1 error $E_r$,
with $t\sim\mathcal{U}(0.05,0.95)$. Figure~\ref{fig:schedule}(a) reports a monotone, large
ordering, $\text{interface}\gg\text{interior}>\text{syntax}\gg\text{special}$ (probe NLL
$D_r=0.11/0.23/0.35$ nats and top-1 error $E_r=4.2/8.8/12.9\%$ for syntax/interior/interface;
special $\approx 0$). The fully trained uniform MDLM reproduces the \emph{same} role
ordering at larger absolute magnitudes (Table~\ref{tab:rolewise_reconstruction}), so it is
the ordering, not the probe's absolute scale, that we rely on.
This heterogeneity is the headroom \method{} exploits.

\section{\method{}}
Motivated by this role heterogeneity, we make the corruption schedule role-dependent.
Although we evaluate it on molecular graphs, the framework applies to any losslessly
serialized graph with identifiable token roles.
Figure~\ref{fig:schedule} previews the construction on real QM9 values.
Let $\Lambda(t)=\int_0^t\beta(s)\,ds$ be the base schedule and $\gamma_r>0$ a
per-role exponent. The absorbing process becomes role-dependent,
\begin{equation}
\alpha_t^{(r)}=1-\exp\!\big(-\gamma_r\,\Lambda(t)\big),
\label{eq:schedule}
\end{equation}
so $\gamma_r<1$ masks role $r$ less (visible longer, revealed earlier) and
$\gamma_r>1$ masks it more. Uniform $\gamma_r\equiv1$ recovers MDLM exactly. The
clean sequence, grammar, and decoder are unchanged, so $D_{\text{SENT}}(z_0)\cong G$
for any grammar-valid $z_0$ (Lemma~\ref{lem:decode}).

\begin{figure}[t]
\centering
\resizebox{\columnwidth}{!}{%
\begin{tikzpicture}[font=\footnotesize,x=1cm,y=1cm]
\draw[black!55] (0,0)--(0,2.30);
\draw[black!55] (0,0)--(4.05,0);
\draw[black!55] (4.05,0)--(4.05,2.30);
\foreach \v/\y in {0/0,0.1/0.6,0.2/1.2,0.3/1.8}{
  \draw[black!40] (-0.06,\y)--(0,\y);
  \node[left,font=\scriptsize,black!60] at (-0.06,\y) {\v};}
\foreach \v/\y in {0/0,5/0.733,10/1.467,15/2.200}{
  \draw[black!40] (4.05,\y)--(4.11,\y);
  \node[right,font=\scriptsize,black!60] at (4.11,\y) {\v};}
\node[rotate=90,measurecol] at (-0.64,1.15) {$D_r$ (NLL)};
\node[rotate=90,stepcol] at (4.80,1.15) {$E_r$ (\%)};
\foreach \xc/\nll/\err/\lab in {0.7/0.000/0.000/Spec, 1.6/0.636/0.616/Syn, 2.5/1.404/1.291/Intr, 3.4/2.118/1.892/Intf}{
  \fill[measurecol!75] ({\xc-0.32},0) rectangle ({\xc-0.02},\nll);
  \fill[stepcol!80]  ({\xc+0.02},0) rectangle ({\xc+0.32},\err);
  \node[below,black!75] at (\xc,0) {\lab};}
\node at (2.0,-0.62) {(a) measured difficulty};
\begin{scope}[xshift=5.7cm]
\draw[->,black!55] (0,0)--(0,2.5);
\draw[black!55] (0,0)--(3.1,0);
\node[rotate=90] at (-0.52,1.25) {$\gamma_r^\star$ ($\eta{=}2$)};
\draw[dashed,black!45] (0,0.85)--(3.1,0.85) node[right,font=\scriptsize,black!55] {$\gamma{=}1$};
\foreach \x/\h/\lab/\c in {0.7/0.44/Intf/traincol, 1.6/1.07/Intr/derivecol, 2.5/2.37/Syn/derivecol}{
  \fill[\c!70] ({\x-0.30},0) rectangle ({\x+0.30},{\h*0.85});
  \node[below,black!75] at (\x,0) {\lab};
  \node[above,font=\scriptsize,black!70] at (\x,{\h*0.85}) {\h};
}
\node at (1.6,-0.58) {(b) derived rate};
\end{scope}
\end{tikzpicture}}
\caption{Measured difficulty and the derived schedule, on real QM9 values.
\textbf{(a)} Measured role difficulty on the probe denoiser: probe NLL $D_r$
(teal, left axis) and top-1 reconstruction error $E_r$ (navy, right axis). Both rank
interface hardest and special trivial.
\textbf{(b)} Risk-optimal rates $\gamma_r^\star$ from Eq.~\ref{eq:smooth} at $\eta{=}2$
(Table~\ref{tab:eta}): the hardest role (interface) is \emph{protected} ($\gamma<1$)
and easier roles corrupted more ($\gamma>1$), at a fixed masking budget. The lossless
decoder is untouched.}
\label{fig:schedule}
\end{figure}
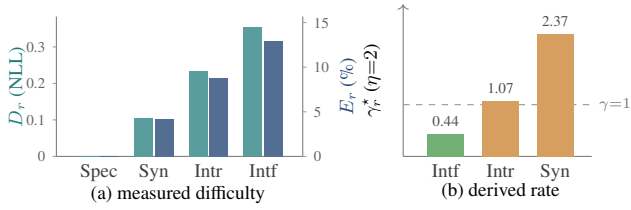

\paragraph{Inverse-exposure loss weighting.}
Protected (low-$\gamma$) roles are masked less often and would otherwise be
under-trained. We reweight the per-token loss by inverse exposure,
\begin{equation}
w(t,r_i)=\frac{1}{\Pr(i\in M_t\mid r_i)+\epsilon},
\label{eq:expweight}
\end{equation}
which is a no-op at $\gamma_r\equiv1$ ($w\equiv1$).

\paragraph{Role-aware sampling.}
Motif roles are not directly observed for masked tokens at generation time, but they
need not be: the denoiser already induces a distribution over the clean token at each
position, and the tokenizer supplies a training token--role prior $P(r\mid v)$. We
therefore estimate a \emph{soft} role distribution
$q_i(r)=\sum_v p_\theta(x_i{=}v\mid z_t)\,P(r\mid v)$ at every masked position $i$ and
carry the same predetermined role rates $\gamma_r$ used in training into the reverse
process, giving each position an expected exponent $\gamma_i=\sum_r q_i(r)\,\gamma_r$
and a role-aware reverse unmasking probability
$p^{\mathrm{unmask}}_i=(a_s^{(i)}-a_t^{(i)})/(1-a_t^{(i)})$ with
$a_t^{(i)}=\exp(-\gamma_i\Lambda(t))$. Positions are then revealed by the combined
confidence-and-schedule score $\log\max_v p_\theta(x_i{=}v\mid z_t)+\log
p^{\mathrm{unmask}}_i$. Algorithm~\ref{alg:sampling} states the full procedure; the
$\gamma_r$ are fixed before sampling, so reverse sampling reuses the derived schedule
end-to-end rather than deferring it to training alone.

\begin{algorithm}[t]
\caption{Role-aware sampling with predetermined role $\gamma$'s}
\label{alg:sampling}
\begin{algorithmic}[1]
\REQUIRE denoiser $p_\theta$; length $L$; reverse steps $\{t\!\to\!s\}$; role rates
  $\gamma_{\text{special}},\gamma_{\text{syntax}},\gamma_{\text{interior}},\gamma_{\text{interface}}$;
  token--role prior $P(r\mid v)$; noise schedule $\Lambda(\cdot)$; reveal counts $k_t$;
  \texttt{[MASK]}, \texttt{[BOS]}, \texttt{[EOS]}; clip floor $\epsilon$
\ENSURE generated graph $\hat G$
\STATE $z\gets[\texttt{[BOS]},\texttt{[MASK]},\dots,\texttt{[MASK]},\texttt{[EOS]}]$
\FOR{each reverse step $t\to s$}
  \STATE $M\gets\{$masked positions in $z\}\setminus\{\text{BOS},\text{EOS}\}$
  \IF{$M=\emptyset$}
    \STATE \textbf{break}
  \ENDIF
  \STATE $\mathrm{probs}\gets p_\theta(\cdot\mid z,t)$ \COMMENT{clean-token probabilities}
  \FORALL{$i\in M$}
    \STATE $q_i(r)\gets\sum_v \mathrm{probs}[i][v]\,P(r\mid v)$ \; for each role $r$
    \STATE $\gamma_i\gets\sum_r q_i(r)\,\gamma_r$
    \STATE $a_t\gets e^{-\gamma_i\Lambda(t)}$,\quad $a_s\gets e^{-\gamma_i\Lambda(s)}$
    \STATE $p^{\mathrm{unmask}}_i\gets\mathrm{clip}\!\big((a_s-a_t)/(1-a_t),\,\epsilon,\,1\big)$
    \STATE $\mathrm{conf}_i\gets\max_v \mathrm{probs}[i][v]$
    \STATE $\mathrm{score}_i\gets\log\mathrm{conf}_i+\log p^{\mathrm{unmask}}_i$
  \ENDFOR
  \STATE $S\gets$ top-$k_t$ positions of $M$ ranked by $\mathrm{score}_i$
  \FORALL{$i\in S$}
    \STATE $z[i]\gets\arg\max_v\mathrm{probs}[i][v]$ \;(greedy)\; \textbf{or}\; sample $\sim\mathrm{probs}[i]$
  \ENDFOR
  \STATE unselected positions remain \texttt{[MASK]}
\ENDFOR
\STATE $\hat G\gets \mathrm{decode}_{\text{SENT/mSENT}}(z)$; sanitize $\hat G$ with RDKit
\RETURN $\hat G$
\end{algorithmic}
\end{algorithm}

\paragraph{Computational overhead.}
\method{} reparameterizes the forward corruption rather than adding a module: identical
$3.32$M parameters, no extra forward passes, and only two $O(L)$ table operations per
training step (the per-token exponent $\gamma_{r_i}$ and the inverse-exposure weight) and
per reverse step, plus a one-time schedule solve over $|\mathcal R|{=}4$ roles. Wall-clock
therefore matches MDLM within run-to-run noise ($\sim\!1.25$ min/epoch; $87.4$ vs.\ $91.8$\,s
per $1000$ samples on one TITAN~RTX). Full comparisons appear in Supplementary Table~S3.

\section{Theoretical Analysis}
\label{sec:theory}

The theory formalizes one claim: if serialized roles differ in reconstruction difficulty
and graph-level impact, uniform masking is not risk-optimal, and a role-aware schedule
should protect high-criticality roles at an unchanged budget. Serialization is lossless
($T_\pi$ a bijection, so $H(Z^\pi)=H(G)$), so schedules matter only because the denoiser is
finite-capacity and roles are not equally easy to reconstruct.

\begin{lemma}[Decoding preserved]\label{lem:decode}
\method{} changes only the forward corruption kernel $q(z_t\mid z_0,r)$; the clean
sequence space, grammar, and decoder are unchanged. Hence, for any grammar-valid
$z_0$, $D_{\text{SENT}}(z_0)\cong G$. Setting $\gamma_r\equiv1$ recovers uniform MDLM
exactly (proof in Supplementary Section~S1).
\end{lemma}

\paragraph{Role-decomposed graph risk.}
Let $I_r$ be the expected graph-level impact of a role-$r$ reconstruction error,
$\epsilon_r$ its residual reconstruction error under a finite-capacity denoiser, and
$\rho_r=\mathbb{E}_t[\alpha_t^{(r)}]$ its average corruption exposure. A role-$r$ token
harms the graph only when corrupted and then misreconstructed, so
\begin{equation}
\mathbb{E}[\Delta_{\mathrm{graph}}]
\le
\mathcal{R}(\rho)
:=
\sum_{r\in\mathcal R}
\pi_r I_r\epsilon_r\rho_r .
\label{eq:risk}
\end{equation}

The bound separates role frequency $\pi_r$, graph impact $I_r$, and reconstruction error
$\epsilon_r$, which we estimate by the measured probe NLL $D_r$.

\paragraph{Role criticality.}
Define the role criticality
\begin{equation}
c_r:=I_r\epsilon_r.
\label{eq:criticality}
\end{equation}
High-criticality roles are both hard to reconstruct and damaging when wrong; the product,
rather than a sum, keeps criticality low unless \emph{both} factors are large, so a role
extreme on only one axis does not dominate the allocation. We use the plug-in estimate
$C_r=\mathrm{Norm}(\widehat D_r)\cdot\mathrm{Norm}(\widehat I_r)$ from measured difficulty
$\widehat D_r$ and impact $\widehat I_r$.

\paragraph{Corruption as constrained allocation.}
Uniform MDLM assigns every role the same exposure $\bar\rho$; we instead minimize the
criticality-weighted exposure at the same budget:
\begin{equation}
\min_{\rho}
\sum_r \pi_r c_r\rho_r
\quad
\text{s.t.}
\quad
\sum_r\pi_r\rho_r=\bar\rho,
\qquad
\rho_{\min}\le\rho_r\le\rho_{\max}.
\label{eq:opt}
\end{equation}

\begin{theorem}[Risk-optimal exposure allocation]\label{thm:main}
Let $c_r:=I_r\epsilon_r\ge0$. Over the fixed-budget feasible set in
Eq.~\ref{eq:opt}, the minimizer has the water-filling form
\[
\rho_r^\star
=
\begin{cases}
\rho_{\min}, & c_r>\lambda \quad(\text{protect}),\\
\rho_{\max}, & c_r<\lambda \quad(\text{corrupt more}),\\
\in[\rho_{\min},\rho_{\max}], & c_r=\lambda,
\end{cases}
\]
for a budget multiplier $\lambda$. Consequently: (i)~$\rho_r^\star$ is non-increasing in
$c_r$; (ii)~uniform MDLM is optimal iff $c_r$ is constant across roles; (iii)~if
criticalities are heterogeneous and the feasible box permits a nonzero exposure shift from a
higher- to a lower-criticality role, then $\mathcal R(\rho^\star)<\mathcal R(\rho^{\mathrm{uni}})$;
and (iv)~since $\rho(\gamma)=\mathbb E_t[1-\exp(-\gamma\Lambda(t))]$ is strictly increasing in
$\gamma$, the optimal exponents $\gamma_r^\star=\rho^{-1}(\rho_r^\star)$ inherit the same
ordering.
\end{theorem}

\noindent\emph{Proof sketch.} The objective is linear on the budget box, so KKT stationarity
gives the water-filling vertex and monotonicity of $\rho(\gamma)$ transfers the ordering to
the exponents; full proof in Supplementary Section~S1.

\begin{corollary}[Uniform null]\label{cor:null}
If $c_r$ is constant the optimum is uniform, so \method{} cannot improve the modeled risk;
conversely, any strict reduction in Eq.~\ref{eq:risk} certifies heterogeneous criticality.
\end{corollary}

\paragraph{Smooth risk-optimal schedule.}
The threshold rule of Theorem~\ref{thm:main} is too sharp for training, so we add a
convex penalty against large deviations from uniform MDLM:
\begin{equation}
\min_{\rho}
\sum_r\pi_r c_r\rho_r
+
\frac{1}{\eta}
\sum_r
\pi_r
\mathrm{KL}(\rho_r\|\bar\rho)
\quad
\text{s.t.}
\quad
\sum_r\pi_r\rho_r=\bar\rho .
\label{eq:regularized-opt}
\end{equation}
\begin{proposition}[Smooth risk-optimal schedule]\label{prop:smooth}
The regularized objective in Eq.~\ref{eq:regularized-opt} is strictly convex on
$(0,1)^{|\mathcal R|}$ and has the unique minimizer
\begin{equation}
\rho_r^\star
=
\sigma\!\left(
\sigma^{-1}(\bar\rho)+\eta(\lambda-c_r)
\right),
\label{eq:smooth}
\end{equation}
where $\lambda$ enforces the budget constraint and $\sigma$ is the logistic sigmoid.
Moreover $\rho_r^\star$ is strictly decreasing in the criticality $c_r$.
\end{proposition}

Higher-criticality roles thus receive lower exposure, and $\eta$ interpolates from uniform
MDLM ($\eta\to0$) to the threshold rule of Theorem~\ref{thm:main} ($\eta\to\infty$),
setting the magnitude of reallocation, not its direction.

Two further guarantees are deferred to Supplementary Section~S1 (Theorems~S1 and~S2, with
all proofs): the smooth allocation is stable to noisy criticality estimates
($\max_r|\widehat\rho_r^\star-\rho_r^\star|=O(\eta\delta)$ for $\max_r|\widehat c_r-c_r|\le\delta$),
and a Rademacher bound shows that at matched architecture, data size, and training budget,
schedule comparisons are governed by the empirical role-weighted risk
$\widehat{\mathcal R}(\rho)=\sum_r\pi_r\widehat I_r\widehat\epsilon_r(\rho_r)$.

\paragraph{Robustness to $\eta$.}
The only free hyperparameter is $\eta$; because Eq.~\ref{eq:smooth} is monotone in $c_r$,
Table~\ref{tab:eta} confirms interface stays protected across datasets and every $\eta>0$.
\begin{table*}[!t]
\centering
\small
\setlength{\tabcolsep}{5pt}
\renewcommand{\arraystretch}{1.08}

\caption{Sensitivity of the derived $\gamma_r$ to $\eta$
(Eq.~\ref{eq:smooth}), re-derived from the measured criticalities
without retraining. For all $\eta>0$, the interface role is protected
($\gamma_r<1$), whereas the interior and syntax roles receive greater
corruption ($\gamma_r>1$). Thus, $\eta$ controls the magnitude of the
allocation rather than its direction. The setting $\eta=0$ corresponds
to uniform MDLM; we use $\eta=2$ (highlighted).}
\label{tab:eta}

\begin{tabular}{@{}l *{9}{c}@{}}
\toprule
& \multicolumn{3}{c}{QM9}
& \multicolumn{3}{c}{MOSES}
& \multicolumn{3}{c}{GuacaMol} \\
\cmidrule(lr){2-4}
\cmidrule(lr){5-7}
\cmidrule(l){8-10}

$\eta$
& Interior & Interface & Syntax
& Interior & Interface & Syntax
& Interior & Interface & Syntax \\
\midrule

$0.0$
& 1.00 & 1.00 & 1.00
& 1.00 & 1.00 & 1.00
& 1.00 & 1.00 & 1.00 \\

$0.5$
& 1.02 & 0.81 & 1.24
& 1.08 & 0.72 & 1.05
& 1.06 & 0.78 & 1.05 \\

$1.0$
& 1.03 & 0.66 & 1.54
& 1.17 & 0.52 & 1.09
& 1.13 & 0.61 & 1.11 \\

\ourrow
$2.0$
& 1.07 & 0.44 & 2.37
& 1.34 & 0.27 & 1.17
& 1.26 & 0.37 & 1.22 \\

$4.0$
& 1.13 & 0.25 & 4.00
& 1.59 & 0.25 & 1.22
& 1.49 & 0.25 & 1.40 \\

$8.0$
& 1.22 & 0.25 & 4.00
& 1.79 & 0.25 & 1.05
& 1.68 & 0.25 & 1.49 \\

\bottomrule
\end{tabular}
\end{table*}
\section{Experimental Setup}
\paragraph{Datasets.}
We evaluate on QM9~\cite{ramakrishnan2014quantum} (molecules up to nine heavy atoms) and
MOSES~\cite{polykovskiy2020molecular} (a larger, more diverse drug-like set), and use
GuacaMol~\cite{brown2019guacamol} as a third benchmark to test whether the derived schedule
and role-wise diagnostics generalize beyond the two generation datasets. All are serialized
losslessly by the motif-aware tokenizer with the per-token role labels that drive the
schedule.

\paragraph{Backbone and protocol.}
All conditions share one masked discrete-diffusion backbone under matched architecture,
optimizer, and training budget, so differences reflect the corruption and sampling schedule
rather than capacity or training duration. \method{} samples with the role-aware confidence
sampler (Algorithm~\ref{alg:sampling}) using the inferred token-role distribution and
role-specific rates, whereas uniform MDLM uses the standard confidence sampler with a single
token-agnostic schedule. Unless noted, we generate $5{,}000$ samples per condition and decode
with the same lossless decoder.

\paragraph{Metrics.}
We report validity (RDKit-sanitizable fraction), uniqueness, novelty, atom stability, QED,
and role-wise reconstruction NLL and top-1 error, with distributional fidelity measured by
Fr\'echet ChemNet Distance (FCD; lower is better) against a \emph{held-out} reference split.
Complete architecture, optimization, schedule-construction, metric, hardware, and seed
details are provided in Supplementary Section~S2.

\section{Results}
\label{sec:results}

We organize the empirical evaluation around the four research questions posed in the
Introduction, treating uniform MDLM as the flat-criticality special case of \method{}
throughout.

\subsection{Are Token Roles Heterogeneous? (RQ1)}
Role-aware corruption is justified only if the token roles are not interchangeable. On a
uniformly masked probe denoiser the roles separate sharply in reconstruction difficulty
(Table~\ref{tab:rolewise_reconstruction}): interface is the hardest role on every dataset
(QM9 NLL $0.497$ and top-1 error $0.166$, against $0.120$ and $0.045$ for syntax; MOSES
interface NLL $0.712$; GuacaMol interface NLL $0.903$), with interior consistently in
between, and Figure~\ref{fig:schedule}(a) shows the same ordering on the probe. The
single-token perturbation impact (Table~\ref{tab:single_token_perturbation}) ranks the roles
in the same direction: interface is the highest-impact role ($I_r=0.759$, 95\% CI
$[0.742,0.776]$), above interior ($0.620$) and syntax ($0.566$), because corrupting a
cross-motif attachment token most often produces an invalid or disconnected graph. The two
axes agree that interface is most critical, though difficulty separates the roles far more
strongly (roughly fourfold) than impact ($\sim\!1.3\times$). Roles thus differ significantly
along both axes on all three datasets, answering RQ1.
\begin{table}[t]
\centering
\caption{Single-token perturbation impact by role (QM9, 2430 trials/role, 95\% CI).
Impact $= P(\text{invalid or disconnected})$. The ordering is
interface $>$ interior $>$ syntax: interface is the highest-impact role,
because corrupting a cross-motif attachment token most often yields an
invalid or disconnected graph.}
\label{tab:single_token_perturbation}

\fontsize{8}{9.5}\selectfont
\setlength{\tabcolsep}{4pt}
\renewcommand{\arraystretch}{1.12}

\begin{tabular}{@{}lccccc@{}}
\toprule
Role &
Impact &
95\% CI &
\\
\midrule

Interface &
$\mathbf{0.759}$ &
$[0.742,\,0.776]$ & \\

Interior &
$0.6198$ &
$[0.605,\,0.632]$ & \\

Syntax &
$0.566$ &
$[0.552,\,0.580]$ & \\

\bottomrule
\end{tabular}
\end{table}

\subsection{Is the Schedule Derived, Not Tuned? (RQ2)}
Given heterogeneous criticalities, does the risk-optimal rule of Theorem~\ref{thm:main}
yield a usable schedule rather than a hand-set one? Table~\ref{tab:eta} instantiates
Eq.~\ref{eq:opt}: from measured difficulty $D_r$ and impact $I_r$ we form
$C_r=\mathrm{Norm}(D_r)\,\mathrm{Norm}(I_r)$ and solve for $\gamma_r$. On QM9, MOSES, and
GuacaMol the top-criticality role is interface, so at every $\eta>0$ the schedule protects it
($\gamma_{\text{interface}}<1$; e.g.\ $0.44$, $0.27$, and $0.37$ respectively at $\eta=2$) and
corrupts the lower-criticality interior and syntax roles more at fixed budget. The allocation
direction agrees across all three datasets, indicating $C_r$ tracks a structural signal rather
than dataset noise. As established in RQ1, difficulty dominates the criticality ranking, so
interface is protected because it is both hardest to reconstruct and highest-impact. The
schedule is thus derived rather than tuned, answering RQ2.

\begin{table}[!htb]
\centering\small
\caption{Role-wise reconstruction performance under uniform evaluation masking
($1000$ graphs/checkpoint). We report negative log-likelihood (NLL) and top-1
reconstruction error for each structural role on QM9, MOSES, and GuacaMol.}
\label{tab:rolewise_reconstruction}
\begin{tabular}{@{}lllcc@{}}
\toprule
Dataset & Method & Role & NLL $\downarrow$ & Top-1 err. $\downarrow$ \\
\midrule

\multirow{6}{*}{QM9}
& \multirow{3}{*}{Uniform MDLM}
  & interior  & 0.261 & 0.095 \\
& & interface & 0.497 & 0.166 \\
& & syntax    & 0.120 & 0.045 \\
\cmidrule(l){2-5}
& \multirow{3}{*}{\method{}}
  & interior  & 0.222 & 0.085 \\
& & interface & 0.354 & 0.132 \\
& & syntax    & 0.097 & 0.039 \\

\midrule

\multirow{6}{*}{MOSES}
& \multirow{3}{*}{Uniform MDLM}
  & interior  & 0.406 & 0.141 \\
& & interface & 0.712 & 0.247 \\
& & syntax    & 0.590 & 0.201 \\
\cmidrule(l){2-5}
& \multirow{3}{*}{\method{}}
  & interior  & 0.390 & 0.136 \\
& & interface & 0.708 & 0.202 \\
& & syntax    & 0.271 & 0.105 \\
\midrule

\multirow{6}{*}{Guacamol}
& \multirow{3}{*}{Uniform MDLM}
  & interior  & 0.528 & 0.180 \\
& & interface & 0.903 & 0.292 \\
& & syntax    & 0.435 & 0.166 \\
\cmidrule(l){2-5}
& \multirow{3}{*}{\method{}}
  & interior  & 0.480 & 0.167 \\
& & interface & 0.899 & 0.286 \\
& & syntax    & 0.314 & 0.118 \\

\bottomrule
\end{tabular}
\end{table}

\subsection{Does It Beat Uniform Masking? (RQ3)}
Theorem~\ref{thm:main} predicts that the heterogeneous criticality established in RQ1 makes
uniform masking suboptimal. Holding architecture, optimizer, training budget, and compute
fixed isolates the corruption and sampling schedule as the only difference from uniform MDLM.
The derived schedule improves validity on all three datasets (QM9 $0.905\!\to\!0.944$,
MOSES $0.920\!\to\!0.938$, GuacaMol $0.787\!\to\!0.841$) and lowers FCD wherever it is
measured (QM9 $1.701\!\to\!1.609$, MOSES $2.125\!\to\!1.850$; FCD and novelty are not
computed on GuacaMol). Uniqueness improves on MOSES and stays high on QM9 and GuacaMol, while
the small novelty dip on QM9 and MOSES reflects the fidelity--diversity trade.
Figure~\ref{fig:multiseed} shows these gains are stable across five seeds, with tight,
largely non-overlapping intervals. Since the
schedule is derived from measured criticalities rather than chosen arbitrarily (RQ2), this
consistent improvement across three datasets is in line with the prediction and reflects
informed structural allocation. A matched-compute staged ablation (Supplementary Table~S2)
separates the mSENT tokenizer gain from the additional gain of the full role-aware
configuration; since that configuration enables the derived schedule and inverse-exposure
weighting jointly, we do not isolate their individual effects.

\begin{figure*}[t]
\centering
\resizebox{\textwidth}{!}{%
\begin{tikzpicture}[font=\footnotesize,x=1cm,y=1cm]
\begin{scope}
  \draw[black!55] (0.15,0)--(5.30,0);
  \draw[black!55] (0.15,0)--(0.15,3.45);
  \foreach \v/\y in {0/0,0.25/0.825,0.5/1.65,0.75/2.475,1/3.3}{
    \draw[black!40] (0.08,\y)--(0.15,\y);
    \node[left,font=\scriptsize,black!60] at (0.08,\y) {\v};}
  \draw[black!55] (5.30,0)--(5.30,3.45);
  \foreach \v/\y in {0/0,1/1.32,2/2.64}{
    \draw[black!40] (5.30,\y)--(5.37,\y);
    \node[right,font=\scriptsize,black!60] at (5.37,\y) {\v};}
  \draw[black!20,dashed] (3.90,0)--(3.90,3.45);
  \foreach \x/\h/\c in {0.63/2.987/0.020, 1.83/3.224/0.013, 3.03/2.577/0.059, 4.43/2.245/0.106}{
    \fill[black!35] (\x-0.14,0) rectangle (\x+0.14,\h);
    \draw[black!70,line width=0.4pt] (\x,\h-\c)--(\x,\h+\c);
    \draw[black!70,line width=0.4pt] (\x-0.05,\h+\c)--(\x+0.05,\h+\c);
    \draw[black!70,line width=0.4pt] (\x-0.05,\h-\c)--(\x+0.05,\h-\c);}
  \foreach \x/\h/\c in {0.97/3.115/0.026, 2.17/3.148/0.010, 3.37/2.429/0.050, 4.77/2.124/0.085}{
    \fill[derivecol!85] (\x-0.14,0) rectangle (\x+0.14,\h);
    \draw[black!70,line width=0.4pt] (\x,\h-\c)--(\x,\h+\c);
    \draw[black!70,line width=0.4pt] (\x-0.05,\h+\c)--(\x+0.05,\h+\c);
    \draw[black!70,line width=0.4pt] (\x-0.05,\h-\c)--(\x+0.05,\h-\c);}
  \foreach \x/\lab in {0.8/Valid, 2.0/Unique, 3.2/Novel, 4.6/FCD}{
    \node[below,font=\scriptsize,black!75] at (\x,-0.04) {\lab};}
  \node[rotate=90,font=\scriptsize,black!65] at (-0.62,1.65) {rate $\uparrow$};
  \node[rotate=90,font=\scriptsize,black!65] at (5.98,1.32) {FCD $\downarrow$};
  \node[font=\small\bfseries] at (2.7,3.72) {QM9};
\end{scope}
\begin{scope}[xshift=7.3cm]
  \draw[black!55] (0.15,0)--(5.30,0);
  \draw[black!55] (0.15,0)--(0.15,3.45);
  \foreach \v/\y in {0/0,0.25/0.825,0.5/1.65,0.75/2.475,1/3.3}{
    \draw[black!40] (0.08,\y)--(0.15,\y);
    \node[left,font=\scriptsize,black!60] at (0.08,\y) {\v};}
  \draw[black!55] (5.30,0)--(5.30,3.45);
  \foreach \v/\y in {0/0,1/1.32,2/2.64}{
    \draw[black!40] (5.30,\y)--(5.37,\y);
    \node[right,font=\scriptsize,black!60] at (5.37,\y) {\v};}
  \draw[black!20,dashed] (3.90,0)--(3.90,3.45);
  \foreach \x/\h/\c in {0.63/3.036/0.026, 1.83/3.119/0.010, 3.03/3.287/0.023, 4.43/2.805/0.127}{
    \fill[black!35] (\x-0.14,0) rectangle (\x+0.14,\h);
    \draw[black!70,line width=0.4pt] (\x,\h-\c)--(\x,\h+\c);
    \draw[black!70,line width=0.4pt] (\x-0.05,\h+\c)--(\x+0.05,\h+\c);
    \draw[black!70,line width=0.4pt] (\x-0.05,\h-\c)--(\x+0.05,\h-\c);}
  \foreach \x/\h/\c in {0.97/3.095/0.010, 2.17/3.254/0.017, 3.37/3.023/0.050, 4.77/2.442/0.042}{
    \fill[derivecol!85] (\x-0.14,0) rectangle (\x+0.14,\h);
    \draw[black!70,line width=0.4pt] (\x,\h-\c)--(\x,\h+\c);
    \draw[black!70,line width=0.4pt] (\x-0.05,\h+\c)--(\x+0.05,\h+\c);
    \draw[black!70,line width=0.4pt] (\x-0.05,\h-\c)--(\x+0.05,\h-\c);}
  \foreach \x/\lab in {0.8/Valid, 2.0/Unique, 3.2/Novel, 4.6/FCD}{
    \node[below,font=\scriptsize,black!75] at (\x,-0.04) {\lab};}
  \node[rotate=90,font=\scriptsize,black!65] at (-0.62,1.65) {rate $\uparrow$};
  \node[rotate=90,font=\scriptsize,black!65] at (5.98,1.32) {FCD $\downarrow$};
  \node[font=\small\bfseries] at (2.7,3.72) {MOSES};
\end{scope}
\begin{scope}[xshift=14.6cm]
  \draw[black!55] (0.15,0)--(5.30,0);
  \draw[black!55] (0.15,0)--(0.15,3.45);
  \foreach \v/\y in {0/0,0.25/0.825,0.5/1.65,0.75/2.475,1/3.3}{
    \draw[black!40] (0.08,\y)--(0.15,\y);
    \node[left,font=\scriptsize,black!60] at (0.08,\y) {\v};}
  \draw[black!55] (5.30,0)--(5.30,3.45);
  \foreach \v/\y in {0/0,1/1.32,2/2.64}{
    \draw[black!40] (5.30,\y)--(5.37,\y);
    \node[right,font=\scriptsize,black!60] at (5.37,\y) {\v};}
  \draw[black!20,dashed] (3.90,0)--(3.90,3.45);
  \foreach \x/\h/\c in {0.63/2.597/0.053, 1.83/3.300/0.000}{
    \fill[black!35] (\x-0.14,0) rectangle (\x+0.14,\h);
    \draw[black!70,line width=0.4pt] (\x,\h-\c)--(\x,\h+\c);
    \draw[black!70,line width=0.4pt] (\x-0.05,\h+\c)--(\x+0.05,\h+\c);
    \draw[black!70,line width=0.4pt] (\x-0.05,\h-\c)--(\x+0.05,\h-\c);}
  \foreach \x/\h/\c in {0.97/2.775/0.026, 2.17/3.247/0.003}{
    \fill[derivecol!85] (\x-0.14,0) rectangle (\x+0.14,\h);
    \draw[black!70,line width=0.4pt] (\x,\h-\c)--(\x,\h+\c);
    \draw[black!70,line width=0.4pt] (\x-0.05,\h+\c)--(\x+0.05,\h+\c);
    \draw[black!70,line width=0.4pt] (\x-0.05,\h-\c)--(\x+0.05,\h-\c);}
  \node[font=\scriptsize\itshape,black!45] at (3.2,0.32) {n/a};
  \node[font=\scriptsize\itshape,black!45] at (4.6,0.32) {n/a};
  \foreach \x/\lab in {0.8/Valid, 2.0/Unique, 3.2/Novel, 4.6/FCD}{
    \node[below,font=\scriptsize,black!75] at (\x,-0.04) {\lab};}
  \node[rotate=90,font=\scriptsize,black!65] at (-0.62,1.65) {rate $\uparrow$};
  \node[rotate=90,font=\scriptsize,black!65] at (5.98,1.32) {FCD $\downarrow$};
  \node[font=\small\bfseries] at (2.7,3.72) {GuacaMol};
\end{scope}
\fill[black!35] (7.20,4.02) rectangle (7.48,4.20);
\node[right,font=\scriptsize,black!80] at (7.53,4.11) {Uniform MDLM};
\fill[derivecol!85] (10.60,4.02) rectangle (10.88,4.20);
\node[right,font=\scriptsize,black!80] at (10.93,4.11) {\method{} (risk-optimal)};
\end{tikzpicture}}
\caption{\textbf{Multi-seed robustness} on QM9, MOSES, and GuacaMol ($5$ seeds, $5000$
samples/seed; whiskers are $95\%$ confidence intervals). Validity, uniqueness, and
novelty (\emph{left} axis, higher is better) and FCD (\emph{right} axis, lower is
better) for uniform MDLM (grey) vs.\ \method{} (amber). \method{} improves validity on
\emph{all three} datasets and FCD on QM9 and MOSES, with tight, largely non-overlapping
intervals; the novelty dip on QM9 and MOSES reflects the fidelity--diversity trade
discussed in the text. Novelty and FCD are not computed on GuacaMol (marked n/a), where
uniqueness is already saturated. FCD is computed against the held-out test reference.
Exact five-seed means and $95\%$ confidence intervals are reported in Supplementary
Table~S1.}
\label{fig:multiseed}
\end{figure*}
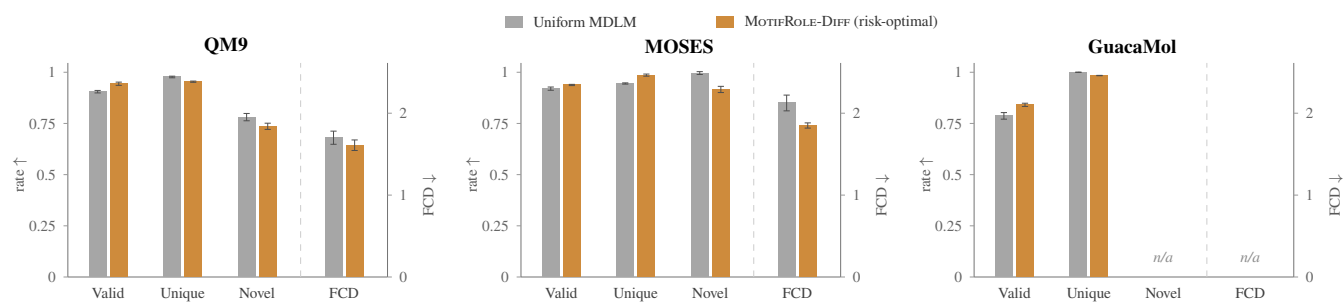

\paragraph{Stability and drug-likeness.}
Table~\ref{tab:stability} completes the metric set: \method{} matches or improves atom
stability (QM9 $0.946\!\to\!0.955$; MOSES $0.902\!\to\!0.915$; GuacaMol $0.796\!\to\!0.839$)
and QED (QM9 $0.419\!\to\!0.455$; MOSES $0.822\!\to\!0.837$; GuacaMol $0.438\!\to\!0.634$),
consistent with a regularizer that preserves difficult-role accuracy, with the largest QED
gain on the more diverse GuacaMol set.

\begin{table}[!htb]
\centering
\small
\setlength{\tabcolsep}{6pt}
\renewcommand{\arraystretch}{1.12}
\caption{Atom stability and QED results on QM9, MOSES, and GuacaMol. We compare Uniform
MDLM and \method{} using $5000$ generated samples.}
\label{tab:stability}
\begin{tabular}{@{}llcc@{}}
\toprule
Dataset & Method & Atom stab.\ $\uparrow$ & QED $\uparrow$ \\
\midrule

QM9
& Uniform MDLM
& $0.946{\scriptscriptstyle\pm.082}$
& $0.419{\scriptscriptstyle\pm.006}$ \\

\rowcolor{bestgreen!18}
& \method{}
& $0.955{\scriptscriptstyle\pm.004}$
& $0.455{\scriptscriptstyle\pm.005}$ \\

\midrule

MOSES
& Uniform MDLM
& $0.902{\scriptscriptstyle\pm.012}$
& $0.822{\scriptscriptstyle\pm.009}$ \\

\rowcolor{bestgreen!18}
& \method{}
& $0.915{\scriptscriptstyle\pm.007}$
& $0.837{\scriptscriptstyle\pm.003}$ \\

\midrule

Guacamol (100k)
& Uniform MDLM
& $0.796{\scriptscriptstyle\pm.002}$
& $0.438{\scriptscriptstyle\pm.005}$ \\

\rowcolor{bestgreen!18}
& \method{}
& $0.839{\scriptscriptstyle\pm.007}$
& $0.634{\scriptscriptstyle\pm.006}$ \\

\bottomrule
\end{tabular}
\end{table}

\paragraph{Where the gain comes from.}
The role-wise diagnostics (Table~\ref{tab:rolewise_reconstruction}) show \method{}
redistributes reconstruction capacity toward high-criticality roles rather than improving
all roles equally. On QM9 all three roles improve, most on the protected interface (NLL
$0.497\!\to\!0.354$, top-1 $0.166\!\to\!0.132$); on MOSES interface NLL is essentially flat
($0.712\!\to\!0.708$) while its top-1 improves ($0.247\!\to\!0.202$) and interior and syntax
improve on both metrics. GuacaMol shows the same pattern: interface stays near baseline
(NLL $0.903\!\to\!0.899$) while the more frequently corrupted syntax role improves sharply
(NLL $0.435\!\to\!0.314$), and interior improves on both metrics. One plausible mechanism is
that protecting the harder interface
tokens, which are corrupted less and therefore remain visible longer during denoising,
improves the context available for reconstructing the remaining interior and syntax tokens,
so the gains are not confined to the protected role. This redistribution is exactly the
reduction in modeled role-weighted risk the theory targets; the FCD, validity, and
uniqueness gains remain empirical downstream effects.

\subsection{Does mSENT Improve Motif Locality? (RQ4)}
\begin{figure}[!htb]
\centering
\begin{tikzpicture}[x=1cm,y=1cm]
  \def\sc{0.55}          
  \def\bh{0.17}          
  \foreach \p/\lab in {0/0,2/2,4/4,6/6}{
    \draw[gray!25] (\p*\sc,0.32) -- (\p*\sc,-4.52);
    \node[gray!70,font=\scriptsize] at (\p*\sc,-4.78) {\lab};
  }
  \draw[gray!55] (0,0.32) -- (0,-4.52);
  \draw[gray!55] (0,-4.52) -- (4.35,-4.52);
  \foreach \y/\w/\name/\pct in {%
      0.0/4.15/Intra-edge distortion/+7.5,
      -0.6/3.99/Boundary alignment/+7.3,
      -1.2/3.67/Same-motif adjacency/+6.7,
      -1.8/2.31/Motif span/+4.2,
      -2.4/1.61/Compactness/+2.9,
      -3.0/0.73/Switch rate/+1.3,
      -3.6/0.55/Fragmentation/+1.0,
      -4.2/0.40/Motif transitions/+0.7}{
    \fill[traincol!85] (0,\y-\bh) rectangle (\w,\y+\bh);
    \node[anchor=east,font=\footnotesize] at (-0.1,\y) {\name};
    \node[anchor=west,font=\scriptsize,traincol!60!black] at (\w+0.08,\y) {\pct\%};
  }
  \node[font=\scriptsize\itshape] at (2.0,-5.15)
    {Relative improvement of mSENT over SENT (\%)};
\end{tikzpicture}
\caption{Motif-aware serialization (mSENT) improves \emph{every} one of the eight
motif-locality metrics over plain SENT on $5{,}000$ tokenized QM9 molecules
(bars sorted by relative gain; each metric oriented so that higher is better). The
largest gains are in intra-motif edge distortion ($22.27\!\to\!20.59$), boundary
alignment ($0.124\!\to\!0.133$), and same-motif adjacency ($0.165\!\to\!0.176$), so
motif-coherent substructures stay contiguous in the token stream. Metric definitions and
complete values are in Supplementary Section~S4 and
Table~S4.}
\label{fig:motif_improve}
\end{figure}
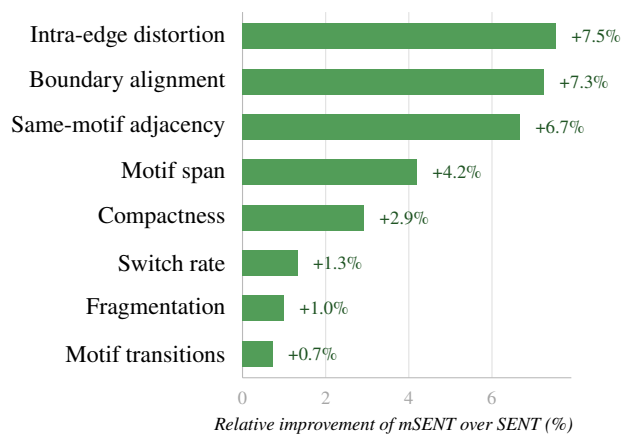
On $5{,}000$ tokenized QM9 molecules, mSENT improves all eight motif-locality metrics over
plain SENT (Fig.~\ref{fig:motif_improve}; example values in the caption), keeping
motif-coherent substructures contiguous in the token stream. Definitions and complete values
are in Supplementary Section~S4 (Table~S4).

\section{Related Work}

\paragraph{Masked and discrete diffusion.}

Discrete diffusion extends denoising diffusion~\cite{ho2020denoising,song2020score} to categorical tokens. D3PM~\cite{austin2021d3pm} introduced structured transition kernels and found the \emph{absorbing} (\texttt{[MASK]}) kernel especially effective for text, with multinomial and argmax diffusion~\cite{hoogeboom2021argmax} as the uniform-kernel counterpart; continuous-time~\cite{campbell2022continuous} and score-based~\cite{lou2023discrete} formulations followed. MDLM~\cite{sahoo2024mdlm} and simplified objectives~\cite{shi2024simplified} show that a clean absorbing-state ELBO recovers much of the quality of autoregressive models while retaining parallel generation, the diffusion counterpart of masked language modelling~\cite{devlin2019bert}. For decoding we adopt confidence-based samplers~\cite{ghazvininejad2019mask,chang2022maskgit}. All of these corrupt every token with a single token-agnostic rate $\alpha_t$; we reparameterize only that rate.

\paragraph{Graph diffusion and serialization.}
Native graph diffusion perturbs adjacency and feature tensors: EDP-GNN~\cite{niu2020permutation}
and GDSS~\cite{jo2022score} use score-based SDEs, DiGress~\cite{vignac2023digress} runs discrete
diffusion over node and edge categories, 3D-aware variants target geometry~\cite{hoogeboom2022equivariant,vignac2023midi},
and autoregressive models generate nodes and edges sequentially~\cite{you2018graph,shi2020graphaf};
these must handle permutation symmetry and quadratic edge tensors. We instead \emph{serialize}
the graph as strings, from SMILES~\cite{weininger1988smiles} and SELFIES~\cite{krenn2020self}
to sequence-model generation~\cite{gomez2018automatic}, while fragment- and motif-based methods
build from substructures~\cite{jin2018junction,jin2020hierarchical}. Our lossless SENT/mSENT
serialization additionally emits a per-token \emph{role} label from the motif partition, so
corruption can depend on structural function while the decoder stays unchanged.


\paragraph{Adaptive corruption schedules.}
Non-uniform noise is established: variational and improved models reshape the \emph{time}
schedule~\cite{kingma2021variational,nichol2021improved}, and continuous relaxations tune
per-dimension noise for discrete data~\cite{dieleman2022continuous}. Closest to us,
DiffusionBERT~\cite{he2023diffusionbert} makes the per-token mask rate non-uniform via corpus
\emph{frequency}; ours is keyed instead to a token's \emph{structural role} and its measured
\emph{difficulty} and \emph{impact}, which frequency cannot see. It is also not hand-designed:
it solves a fixed-budget risk-allocation problem (Theorem~\ref{thm:main}) whose optimum is a
water-filling rule, with uniform MDLM the flat-criticality special case (Corollary~\ref{cor:null}).

\section{Conclusion}
Serialized molecular-graph diffusion is not role-agnostic but \emph{role-sensitive}: the
tokens of a lossless serialization play distinct structural roles that differ systematically
in reconstruction difficulty and in how much an error costs the decoded graph. Because a
uniformly masked probe shows these roles separating sharply, we derive the corruption schedule
as the risk-optimal allocation of a fixed masking budget (Theorem~\ref{thm:main}), with
uniform MDLM its flat-criticality special case. At matched architecture, optimizer, and
training budget, the derived schedule improves distributional fidelity and validity on QM9 and
MOSES, and mSENT improves every motif-locality metric over SENT while leaving the lossless
decoder unchanged. \method{} reparameterizes only the corruption process, adding neither noise
nor compute, so measured structural risk rather than corpus frequency or a hand-set curve
decides where a graph is corrupted; the theorem certifies only the modeled role-weighted risk,
so the generation gains remain empirical.

\section{Limitations and Broader Impact}
Our single-token impact estimator may understate multi-token interface fragility; the
inference-time role prior $P(r\mid v)$ could be sharpened with context-conditioned estimation;
and the full-scale MOSES run transfers the QM9-derived schedule rather than re-estimating it,
so a per-dataset schedule can only help while absolute FCD still trails full-training
graph-native baselines. Theorem~\ref{thm:main} is risk-optimal only for the modeled objective
and does not guarantee that every generation metric improves. On broader impact, \method{}
only reparameterizes how a fixed serialization is corrupted and adds no capability beyond the
uniform-MDLM baseline; as with any molecular generator, improved sample quality carries a
general dual-use risk. All datasets are public and contain no sensitive information.

\bibliography{references}

\appendix
\onecolumn

\setcounter{secnumdepth}{2}
\setcounter{section}{0}
\setcounter{figure}{0}
\setcounter{table}{0}
\setcounter{equation}{0}
\setcounter{theorem}{0}
\setcounter{lemma}{0}
\setcounter{corollary}{0}
\setcounter{proposition}{0}
\setcounter{definition}{0}
\setcounter{algorithm}{0}
\renewcommand{\thesection}{S\arabic{section}}
\renewcommand{\thesubsection}{S\arabic{section}.\arabic{subsection}}
\renewcommand{\thefigure}{S\arabic{figure}}
\renewcommand{\thetable}{S\arabic{table}}
\renewcommand{\theequation}{S\arabic{equation}}
\renewcommand{\thetheorem}{S\arabic{theorem}}
\renewcommand{\thelemma}{S\arabic{lemma}}
\renewcommand{\thecorollary}{S\arabic{corollary}}
\renewcommand{\theproposition}{S\arabic{proposition}}
\renewcommand{\thedefinition}{S\arabic{definition}}
\makeatletter
\renewcommand{\thealgorithm}{S\arabic{algorithm}}
\makeatother

\noindent This document collects the technical appendices for the main paper. Section,
figure, and equation numbers are prefixed with ``S''; references to Lemmas, Theorems,
Equations, Tables, and Algorithms without an ``S'' prefix refer to the main paper.

\section{Proofs and Additional Theoretical Results}
\label{sec:supp_proofs}

Throughout, $\pi_r>0$ denotes the fraction of role-$r$ tokens,
$c_r:=I_r\epsilon_r\ge0$ denotes role criticality, and
$\bar\rho\in(\rho_{\min},\rho_{\max})$ denotes the uniform-MDLM exposure budget.

\begin{proof}[Information preservation]
We show the identity $H(Z^\pi)=H(G)$ stated inline in the main text.
Because $T_\pi$ is deterministic and invertible with inverse decoder $D$, it is a
bijection between the support of $G$ and the support of $Z^\pi=T_\pi(G)$. Therefore,
for every graph $g$,
\[
\Pr(Z^\pi=T_\pi(g))=\Pr(G=g).
\]
The two random variables have the same multiset of probability masses, only relabeled.
Shannon entropy is invariant under deterministic bijections, so
$H(Z^\pi)=H(G)$.
\end{proof}

\begin{proof}[Proof of Lemma~1 (Decoding preserved)]
\method{} modifies only the forward corruption process:
\[
\alpha_t^{(r)}=1-\exp(-\gamma_r\Lambda(t)).
\]
It does not modify the clean sequence distribution, the grammar, or the SENT decoder.
The reverse model still predicts clean tokens $z_0^{(i)}$. Hence, whenever the
generated clean sequence $z_0$ is grammar-valid, the lossless decoder reconstructs the
same graph object, $D_{\text{SENT}}(z_0)\cong G$. Setting $\gamma_r\equiv1$ recovers
the uniform MDLM corruption process exactly.
\end{proof}

\begin{proof}[Derivation of Eq.~5]
For a token of role $r$, graph distortion occurs only when the token is corrupted and
then incorrectly reconstructed. Averaged over time, the probability of corruption is
$\rho_r$. \emph{Conditional on corruption}, the residual probability that the
finite-capacity denoiser reconstructs the token incorrectly is $\epsilon_r$ (estimated
from the probe difficulty $D_r$), and conditional on such an error the expected
graph-level distortion is at most $I_r$. Multiplying these factors and weighting by the
role frequency $\pi_r$, the expected contribution of role $r$ is bounded by
\[
\pi_r\, I_r\, \epsilon_r\, \rho_r .
\]
Summing over roles gives the role-decomposed graph risk of the main text,
\[
\mathbb{E}[\Delta_{\mathrm{graph}}]
\le
\sum_r
\pi_r\, I_r\, \epsilon_r\, \rho_r
=:
\mathcal{R}(\rho),
\]
which is Eq.~5 of the main paper and the exact objective minimized by Theorem~1. Here
$\epsilon_r$ is the residual error \emph{conditional on corruption} and does not depend
on $\rho_r$; the exposure enters only through the linear factor $\rho_r$.

\medskip
\noindent\emph{Remark (exposure-dependent form).} If the residual error itself varies
with exposure, $\epsilon_r=\epsilon_r(\rho_r)$, the same argument yields the more general
bound $\mathbb{E}[\Delta_{\mathrm{graph}}]\le\sum_r\pi_r I_r\epsilon_r(\rho_r)$. The linear
objective above is the special case in which $\epsilon_r$ is exposure-independent; the
generalization result (Theorem~S2) is stated for this exposure-dependent form.
\end{proof}

\begin{proof}[Proof of Theorem~1]
We minimize
\[
\mathcal R(\rho)=\sum_r\pi_r c_r\rho_r
\]
over the compact feasible set
\[
P=
\left\{
\rho:
\sum_r\pi_r\rho_r=\bar\rho,\;
\rho_{\min}\le\rho_r\le\rho_{\max}
\right\}.
\]
Since $P$ is nonempty and compact and $\mathcal R$ is continuous, a minimizer exists.

Introduce a Lagrange multiplier $\lambda$ for the budget constraint and nonnegative
multipliers $\underline\mu_r,\overline\mu_r$ for the lower and upper box constraints.
The Lagrangian is
\begin{align}
\mathcal L(\rho,\lambda,\underline\mu,\overline\mu)
&=
\sum_r \pi_r c_r \rho_r
-\lambda\left(
\sum_r \pi_r \rho_r-\bar\rho
\right) \notag\\
&\quad
-\sum_r \underline\mu_r(\rho_r-\rho_{\min})
+\sum_r \overline\mu_r(\rho_r-\rho_{\max}).
\end{align}
Stationarity gives
\[
\pi_r(c_r-\lambda)-\underline\mu_r+\overline\mu_r=0.
\]
If $c_r>\lambda$, then stationarity requires
$\underline\mu_r>0$, so complementary slackness gives
$\rho_r^\star=\rho_{\min}$. If $c_r<\lambda$, then
$\overline\mu_r>0$, so $\rho_r^\star=\rho_{\max}$. If
$c_r=\lambda$, both box multipliers may vanish and any value in the interval is
allowed. This proves the water-filling form and monotonicity.

If $c_r\equiv c$, then
\[
\mathcal R(\rho)
=
c\sum_r\pi_r\rho_r
=
c\bar\rho,
\]
so every feasible allocation, including uniform MDLM, is optimal.

Conversely, suppose there exist roles $a,b$ with $c_a<c_b$. Because
$\bar\rho\in(\rho_{\min},\rho_{\max})$, choose
\[
0<\delta\le
\min\{\pi_a(\rho_{\max}-\bar\rho),
      \pi_b(\bar\rho-\rho_{\min})\}.
\]
Move exposure from the higher-criticality role $b$ to the lower-criticality role $a$:
\[
\rho_a=\bar\rho+\delta/\pi_a,
\qquad
\rho_b=\bar\rho-\delta/\pi_b,
\]
and keep all other roles at $\bar\rho$. This preserves the budget and remains inside
the box. The change in risk is
\[
\mathcal R(\rho)-\mathcal R(\rho^{\mathrm{uni}})
=
c_a\delta-c_b\delta
=
(c_a-c_b)\delta
<0.
\]
Therefore uniform MDLM is not optimal whenever criticalities are heterogeneous and a
feasible exposure shift exists.

Finally,
\[
\rho(\gamma)
=
\mathbb E_t[1-\exp(-\gamma\Lambda(t))]
\]
is strictly increasing because
\[
\frac{\partial}{\partial\gamma}
\left(1-\exp(-\gamma\Lambda(t))\right)
=
\Lambda(t)\exp(-\gamma\Lambda(t))
>
0
\]
whenever $\Lambda(t)>0$. Thus $\rho^{-1}$ exists on the feasible interval, and the
ordering of optimal exposures transfers to the ordering of optimal exponents
$\gamma_r^\star=\rho^{-1}(\rho_r^\star)$.
\end{proof}

\begin{proof}[Proof of Corollary~1]
The result follows immediately from Theorem~1. If all criticalities are
equal, the objective is constant on the fixed-budget feasible set, so uniform MDLM is
optimal. If any strict improvement over uniform is possible, the criticalities cannot
all be equal.
\end{proof}

\begin{proof}[Proof of Proposition~1]
Consider the regularized objective
\[
J(\rho)
=
\sum_r\pi_r c_r\rho_r
+
\frac{1}{\eta}
\sum_r\pi_r\mathrm{KL}(\rho_r\|\bar\rho).
\]
For Bernoulli means,
\[
\mathrm{KL}(\rho\|\bar\rho)
=
\rho\log\frac{\rho}{\bar\rho}
+
(1-\rho)\log\frac{1-\rho}{1-\bar\rho}.
\]
Its second derivative is
\[
\frac{d^2}{d\rho^2}\mathrm{KL}(\rho\|\bar\rho)
=
\frac{1}{\rho(1-\rho)}
>
0,
\]
so the objective is strictly convex on $(0,1)$. Hence the constrained minimizer is
unique.

The derivative is
\[
\frac{d}{d\rho}\mathrm{KL}(\rho\|\bar\rho)
=
\sigma^{-1}(\rho)-\sigma^{-1}(\bar\rho),
\]
where $\sigma^{-1}$ is the logit. The Lagrangian stationarity condition is
\[
c_r
+
\frac{1}{\eta}
\left(
\sigma^{-1}(\rho_r)-\sigma^{-1}(\bar\rho)
\right)
-\lambda
=
0.
\]
Solving for $\rho_r$ gives
\[
\sigma^{-1}(\rho_r^\star)
=
\sigma^{-1}(\bar\rho)+\eta(\lambda-c_r),
\]
and therefore
\[
\rho_r^\star
=
\sigma\!\left(
\sigma^{-1}(\bar\rho)+\eta(\lambda-c_r)
\right).
\]
Since the sigmoid is strictly increasing and the argument is strictly decreasing in
$c_r$, the optimal exposure is strictly decreasing in role criticality.
\end{proof}

The following two guarantees are stated informally in the main text (end of the
Theoretical Analysis section) and are formalized and proved here as supplement-only
results.

\begin{theorem}[Stability to noisy criticality estimates]\label{thm:supp_stability}
Let $\widehat c_r$ be plug-in criticality estimates with
$\max_r|\widehat c_r-c_r|\le\delta$. Then the smooth allocation of Eq.~9 satisfies
$\max_r|\widehat\rho_r^\star-\rho_r^\star|=O(\eta\delta)$, and if $\rho^{-1}$ is locally
Lipschitz on the feasible interval, $\max_r|\widehat\gamma_r^\star-\gamma_r^\star|=O(\eta\delta)$.
\end{theorem}

\begin{proof}
For the smooth allocation,
\[
\rho_r^\star(c)
=
\sigma\!\left(
\sigma^{-1}(\bar\rho)+\eta(\lambda-c_r)
\right).
\]
The sigmoid derivative satisfies
\[
0<\sigma'(x)\le\frac14.
\]
Therefore, for fixed $\lambda$,
\[
\left|
\rho_r^\star(\widehat c)-\rho_r^\star(c)
\right|
\le
\frac{\eta}{4}|\widehat c_r-c_r|.
\]
If the multiplier also changes from $\lambda$ to $\widehat\lambda$, then
\[
\left|
\widehat\rho_r^\star-\rho_r^\star
\right|
\le
\frac{\eta}{4}
\left(
|\widehat c_r-c_r|
+
|\widehat\lambda-\lambda|
\right).
\]
The budget multiplier is determined implicitly by
\[
\sum_r\pi_r
\sigma\!\left(
\sigma^{-1}(\bar\rho)+\eta(\lambda-c_r)
\right)
=
\bar\rho.
\]
The left-hand side is smooth and monotone in $\lambda$. By the implicit function
theorem, perturbing all $c_r$ by at most $\delta$ perturbs $\lambda$ by $O(\delta)$.
Thus
\[
\max_r|\widehat\rho_r^\star-\rho_r^\star|
=
O(\eta\delta).
\]
If $\rho^{-1}$ is locally Lipschitz on the feasible interval, then
\[
|\widehat\gamma_r^\star-\gamma_r^\star|
=
|\rho^{-1}(\widehat\rho_r^\star)-\rho^{-1}(\rho_r^\star)|
=
O(\eta\delta).
\]
\end{proof}

\begin{theorem}[Schedule generalization]\label{thm:supp_generalization}
Assume the per-example graph-risk loss is bounded by $B$. With probability at least
$1-\delta$, for all $p_\theta\in\mathcal F$,
$\mathbb E[\Delta_{\mathrm{graph}}]\le\widehat{\mathcal R}(\rho)+2B\mathfrak R_N(\mathcal F)+B\sqrt{\log(1/\delta)/(2N)}$,
where $\widehat{\mathcal R}(\rho)=\sum_r\pi_r\widehat I_r\widehat\epsilon_r(\rho_r)$. At matched
architecture, data size, and training budget the complexity terms are shared, so schedule
comparisons are governed by the empirical role-weighted risk $\widehat{\mathcal R}(\rho)$.
\end{theorem}

\begin{proof}
For a fixed exposure allocation $\rho$, define the population role-weighted graph
risk
\[
\mathcal R(\rho)
=
\sum_r\pi_r I_r\epsilon_r(\rho_r).
\]
The empirical counterpart is
\[
\widehat{\mathcal R}(\rho)
=
\sum_r\pi_r\widehat I_r\widehat\epsilon_r(\rho_r).
\]
Assume the per-example graph-risk loss is bounded by $B$. Standard uniform
convergence using Rademacher complexity gives, with probability at least
$1-\delta$, for all $p_\theta\in\mathcal F$,
\[
\mathcal R(\rho)
\le
\widehat{\mathcal R}(\rho)
+
2B\mathfrak R_N(\mathcal F)
+
B\sqrt{\frac{\log(1/\delta)}{2N}}.
\]
Combining this with the exposure-dependent form of Eq.~5 (see the Remark above) gives
\[
\mathbb E[\Delta_{\mathrm{graph}}]
\le
\widehat{\mathcal R}(\rho)
+
2B\mathfrak R_N(\mathcal F)
+
B\sqrt{\frac{\log(1/\delta)}{2N}}.
\]
When comparing schedules at matched architecture, data size, and training budget, the
complexity terms are shared. Therefore differences in the bound are driven by the
empirical role-weighted risk term.
\end{proof}

\section{Reproducibility and Experimental Details}
\label{sec:supp_reproducibility}
All results use existing checkpoints with no per-experiment retraining beyond the
trained models described below. Every generation script is deterministic given its seed:
each run is seeded once at start-up, and the multi-seed results
(Table~\ref{tab:supp_multiseed}) average over five fixed generation seeds held constant
across both methods, so \method{} and uniform MDLM are compared on matched seeds.

\paragraph{Backbone and training.}
A single masked discrete-diffusion transformer is shared by all conditions:
hidden size $256$, $4$ layers, $\sim\!3.32$M parameters, absorbing (\texttt{[MASK]})
forward process. Optimizer Adam, learning rate $3\times10^{-4}$, batch size $64$. Only the per-role exponents $\gamma_r$ and
the inverse-exposure weight differ across schedules; the architecture, optimizer, and budget are
held fixed (matched compute).

\paragraph{Schedule derivation.}
Role criticalities $C_r$ are measured once on the uniform-MDLM checkpoint
(difficulty $D_r$ from probe NLL, impact $I_r$ from single-token perturbation) and the
smooth schedule (Eq.~9) is solved with $\eta{=}2$ and $\gamma_r$ clipped to
$[0.25,4.0]$; special is fixed at $\gamma{=}1$.

\paragraph{Metrics.}
Validity is the RDKit-sanitizable fraction; uniqueness and novelty use canonical
SMILES (novelty vs.\ the training set). QED follows~\cite{bickerton2012quantifying}.
FCD~\cite{preuer2018frechet} is computed with \texttt{fcd\_torch} against a
\emph{held-out} reference split, never the training set.

\paragraph{Hardware.}
All timing and generation runs use a single NVIDIA TITAN~RTX ($24$\,GB); no multi-GPU
or mixed-precision tricks are required. Code, configs, and the per-experiment scripts will be
released publicly upon publication.

\section{Additional Experimental Results and Ablations}
\label{sec:supp_results}

\paragraph{Multi-seed robustness (numeric values for Figure~2 of the main paper).}
Table~\ref{tab:supp_multiseed} gives the exact per-metric means and $95\%$ confidence
intervals summarized as grouped bars in the main paper's multi-seed figure.

\begin{table*}[!htb]
\centering
\small
\setlength{\tabcolsep}{5pt}
\renewcommand{\arraystretch}{1.18}

\caption{Multi-seed robustness on QM9, MOSES, and GuacaMol. Values are mean~$\pm$~$95\%$
confidence interval over $5$ generation seeds using $5000$ samples/seed; FCD is computed
against the held-out test reference. Novelty and FCD are not computed on GuacaMol (marked
``--''). Best values per dataset are highlighted.}
\label{tab:supp_multiseed}

\begin{tabular}{@{}llcccc@{}}
\toprule
Dataset & Method & Valid $\uparrow$ & Unique $\uparrow$ &
Novel $\uparrow$ & FCD $\downarrow$ \\
\midrule

QM9
& Uniform MDLM
& $0.905 \pm 0.006$
& $0.977 \pm 0.004$
& $0.781 \pm 0.018$
& $1.701 \pm 0.080$ \\

\rowcolor{bestgreen!18}
& \method{}
& $0.944 \pm 0.008$
& $0.954 \pm 0.003$
& $0.736 \pm 0.015$
& $1.609 \pm 0.064$ \\

\addlinespace[2pt]
\midrule
\addlinespace[2pt]

MOSES
& Uniform MDLM
& $0.920 \pm 0.008$
& $0.945 \pm 0.003$
& $0.996 \pm 0.007$
& $2.125 \pm 0.096$ \\

\rowcolor{bestgreen!18}
& \method{}
& $0.938 \pm 0.003$
& $0.986 \pm 0.005$
& $0.916 \pm 0.015$
& $1.850 \pm 0.032$ \\

\addlinespace[2pt]
\midrule
\addlinespace[2pt]

GuacaMol (100k)
& Uniform MDLM
& $0.787 \pm 0.016$
& $1.000 \pm 0.000$
& --
& -- \\

\rowcolor{bestgreen!18}
& \method{}
& $0.841 \pm 0.008$
& $0.984 \pm 0.001$
& --
& -- \\

\bottomrule
\end{tabular}
\end{table*}

\begin{table}[t]
\centering
\caption{Matched-compute \emph{staged} component ablation on QM9. Row~1$\to$Row~2 isolates
the mSENT tokenizer contribution (SENT $\to$ mSENT, both uniform, reweighting off); Row~2$\to$Row~3
gives the additional contribution of the \emph{complete} role-aware configuration (derived
schedule \emph{and} inverse-exposure weighting together). Because the final configuration turns
on both components jointly, this table does not isolate their individual effects; a
schedule-only (reweighting-off) row would be required for a full decomposition. Best available
value for each metric is shown in bold.}
\label{tab:supp_ablation}

\fontsize{7.2}{8.6}\selectfont
\setlength{\tabcolsep}{1.8pt}
\renewcommand{\arraystretch}{1.15}

\begin{tabular}{@{}cllccccc@{}}
\toprule
\# &
Tokenizer &
Schedule &
\shortstack{Reweight-\\ing} &
Valid. $\uparrow$ &
Unique. $\uparrow$ &
Novel. $\uparrow$ &
\shortstack{Atom\\Stab. $\uparrow$} \\
\midrule

1 &
SENT &
Uniform ($\gamma=1$) &
Off &
$0.853$ &
$0.750$ &
$0.733$ &
$0.887$ \\

2 &
mSENT &
Uniform ($\gamma=1$) &
Off &
$0.905$ &
$\mathbf{0.977}$ &
$\mathbf{0.781}$ &
$0.946$ \\

3 &
mSENT &
MotifRole ($\gamma^{*}$) &
On &
$0.944$ &
$0.954$ &
$0.736$ &
$\mathbf{0.955}$ \\

\bottomrule
\end{tabular}
\end{table}

\paragraph{Computational overhead.}
Table~\ref{tab:supp_overhead} gives the wall-clock and parameter counts referenced in the
main paper's ``Computational overhead'' paragraph: \method{} adds no parameters and no
extra passes, and matches uniform MDLM in training and sampling time within run-to-run
noise.

\begin{table}[!htb]
\centering
\small
\caption{Computational overhead of \method{} vs.\ uniform MDLM (QM9, $4$-layer backbone,
one TITAN~RTX). The schedule adds no parameters and no passes; wall-clock matches within
noise. Role-aware sampling (Algorithm~1 of the main paper) adds only two $O(L)$ table
operations per reverse step.}
\label{tab:supp_overhead}
\setlength{\tabcolsep}{5pt}
\begin{tabular}{@{}lccc@{}}
\toprule
Method & Params & Train (min/epoch) & Sample $1000$ (s) \\
\midrule
Uniform MDLM        & $3.32$M & $\sim\!1.25$ & $91.8$ \\
\ourrow
\method{}           & $3.32$M & $\sim\!1.25$ & $87.4$ \\
\bottomrule
\end{tabular}
\end{table}

\section{Motif-Preservation Metric Definitions and Full Results}
\label{sec:supp_motif_metrics}
We compare SENT and mSENT on $5{,}000$ tokenized QM9 molecules. For each molecule we
compute atom-level motif assignments with the CAMT5-style motif partition, then
serialize the same graph with either SENT or mSENT. Let $\pi(v)$ be the first sequence
position at which atom $v$ appears and $c(v)$ the motif ID of $v$. Lower span, motif
transitions, switch rate, intra-edge distortion, and fragmentation indicate better motif
locality; higher compactness, same-motif adjacency, and boundary alignment indicate better
motif preservation. The per-molecule metrics, listed in the same order as the rows of
Table~\ref{tab:supp_motif_metrics}, are:
\begin{align}
\mathrm{Span}(m)
&= \max_{v\in V_m}\pi(v)
 - \min_{v\in V_m}\pi(v) + 1 ,
\label{eq:motif_span}
\\[2pt]
\mathrm{Compact}(m)
&= \frac{|V_m|}{\mathrm{Span}(m)} ,
\label{eq:motif_compact}
\\[2pt]
\mathrm{SMAR}
&=
\frac{
\sum_{i=1}^{L-1}
\mathbf{1}\!\left[c(z_i)=c(z_{i+1})\right]
}{
L-1
} ,
\label{eq:smar}
\\[2pt]
\mathrm{Trans}
&=
\sum_{i=1}^{L-1}
\mathbf{1}\!\left[c(z_i)\neq c(z_{i+1})\right] ,
\label{eq:transitions}
\\[2pt]
\mathrm{Switch}
&=
\frac{\mathrm{Trans}}{L-1} ,
\label{eq:switch}
\\[2pt]
\mathrm{IBA}
&=
\frac{
\sum_{i=1}^{L-1}
\mathbf{1}[c(z_i)\neq c(z_{i+1})]\,
\mathbf{1}[(z_i,z_{i+1})\in E]
}{
\mathrm{Trans}
} ,
\label{eq:iba}
\\[2pt]
D_{\mathrm{intra}}
&=
\frac{1}{|E_{\mathrm{intra}}|}
\sum_{(u,v)\in E_{\mathrm{intra}}}
|\pi(u)-\pi(v)| ,
\label{eq:dintra}
\\[2pt]
\mathrm{Frag}(m)
&=
\#\{\text{contiguous blocks of motif } m\}.
\label{eq:frag}
\end{align}
Here $E_{\mathrm{intra}}=\{(u,v)\in E : c(u)=c(v)\}$ denotes the set of intra-motif
edges (both endpoints in the same motif). Equations~\ref{eq:motif_span}--\ref{eq:frag}
correspond one-to-one, in order, to the mean motif span, compactness, same-motif adjacency,
motif transitions, switch rate, boundary alignment, intra-edge distortion, and fragmentation
rows of Table~\ref{tab:supp_motif_metrics}.

Table~\ref{tab:supp_motif_metrics} reports the full per-metric values summarized
as relative gains in Figure~3 of the main paper; arrows indicate
the preferred direction, and mSENT improves every metric.

\begin{table}[!htb]
\centering
\small
\setlength{\tabcolsep}{6pt}
\renewcommand{\arraystretch}{1.12}
\caption{Serialization-level motif preservation metrics comparing SENT and mSENT on QM9.
Arrows indicate the preferred direction for each metric; mSENT improves all eight.}
\label{tab:supp_motif_metrics}
\begin{tabular}{@{}lcc@{}}
\toprule
Metric & SENT & mSENT \\
\midrule
Mean motif span $\downarrow$              & 18.09 & 17.33 \\
Compactness $\uparrow$                    & 0.686 & 0.706 \\
Same-motif adjacency $\uparrow$           & 0.165 & 0.176 \\
Motif transitions $\downarrow$            & 22.03 & 21.87 \\
Switch rate $\downarrow$                  & 0.835 & 0.824 \\
Boundary alignment $\uparrow$             & 0.124 & 0.133 \\
Intra-edge distortion $\downarrow$        & 22.27 & 20.59 \\
Fragmentation $\downarrow$                & 1.796 & 1.778 \\
\bottomrule
\end{tabular}
\end{table}

\end{document}